%% file: main.tex
\documentclass{hld2024} % Include author names
\usepackage{subcaption}
\usepackage{makecell}
\usepackage{mathtools}
\usepackage{caption}
\DeclarePairedDelimiter{\abs}{\lvert}{\rvert}
\usepackage{multicol, multirow}
\newcommand{\ltwo}{$\ell_2$}

\usepackage{graphicx}
\usepackage{booktabs}

\title[The Hidden Pitfalls of the Cosine Similarity Loss]{The Hidden Pitfalls of the Cosine Similarity Loss}

\hldauthor{%
\Name{Andrew Draganov} \Email{draganovandrew@cs.au.dk}\\
\addr {Aarhus University} 
\AND
\Name{Sharvaree Vadgama} \Email{sharvaree.vadgama@gmail.com}\\
\addr {AMLab, University of Amsterdam}
\AND
\Name{Erik J. Bekkers} \Email{e.j.bekkers@uva.nl}\\
\addr {AMLab, University of Amsterdam} }

\begin{document}

\maketitle

\begin{abstract}%
\input{Sections/0_abstract}
\end{abstract}

%\begin{keywords}%
%  List of keywords%
%\end{keywords}

\input{Sections/1_intro}
\input{Sections/2_related_work}
\input{Sections/3_embedding_analysis}
\input{Sections/4_verification}
\input{Sections/5_experiments}

\bibliography{references}
\newpage
\clearpage
\appendix

%\section{More}
\input{Sections/9_appendix}

\end{document}

%% file: Sections/0_abstract.tex
We show that the gradient of the cosine similarity between two points goes to zero in two under-explored settings: (1) if a point has large magnitude or (2) if the
points are on opposite ends of the latent space. Counterintuitively, we prove that optimizing the cosine similarity between points forces them to \emph{grow} in
magnitude. Thus, (1) is unavoidable in practice. We then observe that these derivations are extremely general -- they hold across deep learning architectures
and for many of the standard self-supervised learning (SSL) loss functions. This leads us to propose \emph{cut-initialization}: a simple change to network
initialization that helps all studied SSL methods converge faster.

%% file: Sections/1_intro.tex
\section{Introduction}
Self-supervised learning (SSL) methods can learn robust, multi-use representations in the absence of labeled data by ensuring that similar (positive) inputs are mapped to nearby locations in embedding space. Despite their success, we find that the properties of the cosine similarity -- the common loss function across SSL methods -- have gone under-studied.

% \begin{figure}[t!]
%     \centering
%     \includegraphics[trim={0.3cm 5cm 2cm 0},clip,width=\linewidth]{Images/paper_overview.pdf}
%     \caption{(\emph{Left}) The SSL embedding pipeline: given two views of the same input, we attract their representations by maximizing the cosine similarity. (\emph{Middle}) Depiction of the embedding-norm effect on latent embedding $z_i$: the effective learning rate decreases as $||z_i||$ grows. (\emph{Right}) Depiction of the opposite-halves effect on $z_i$: the effective learning rate decreases as the angle between $z_i$ and $z_j$ goes to $\pi$.}
%     \label{fig:overview} \vspace*{-0.5cm}
% \end{figure}

Indeed, we verify that the cosine similarity loss has unexpected interactions with gradient descent. Namely, the gradient acting on a point is near-zero if the point has large magnitude or if it is on the opposite side of the space from its counterpart. We prove that both cases impose at least a quadratic slowdown on the gradient descent convergence. Furthermore, we show that optimizing the cosine similarity between two points \emph{must} cause those points to grow in magnitude.

This leads to an unfortunate catch-22: SSL methods can only be trained under small embedding norms but optimizing the loss-function grows those same embedding norms. Furthermore, this derivation holds for any objective that is itself a function of the cosine similarity (such as the InfoNCE loss~\cite{contr_pred_coding}) and is \emph{independent} of the SSL architecture, implying that almost all SSL models are affected. We experimentally verify that this is a concern in practice by showing that large embedding norms indeed slow down SSL convergence across architectures and training paradigms. We therefore propose \emph{cut-initialization} to mitigate the issue and verify that, when paired with \ltwo-normalization, cut-initialization improves convergence across all settings.

%% file: Sections/2_related_work.tex
\section{Preliminaries and Related Work}

SSL representations are learned from a dataset without utilizing any labels. This is typically done by obtaining two augmented variants $x_i$ and $x_j$ from an input image $x$ and ensuring that the corresponding embeddings $z_i$ and $z_j$ have high cosine similarity. We refer to $(x_i, x_j)$ (resp. $(z_i, z_j)$) as `positive' pairs of points (resp. embeddings). Each SSL method then introduces an additional mechanism that prevents the entire learned representation from collapsing to a single embedding.

\paragraph{SSL methods.} SimCLR~\cite{simclr} follows a line of research of contrastive methods for self-supervised learning \cite{contr_pred_coding, data_efficient_cpc, moco, mocov2, pirl} which all use repulsions from \emph{negative} samples, $z_k$, to prevent collapse. These negative samples are the other samples in the batch, implying that no two inputs can be mapped to the same location. Thus, SimCLR minimizes the InfoNCE loss \cite{contr_pred_coding} for embedding $z_i$:
\begin{equation}
    \label{eq:infonce}
    \mathcal{L}_i(\mathbf{Z}) = -\log \frac{\text{sim}(z_i, z_j)}{\sum_{k \not\sim i} \text{sim}(z_i, z_k)} = -\underbrace{\hat{z}_i^\top \hat{z}_j}_{\mathcal{L}^\mathcal{A}_i(\mathbf{Z})} + \underbrace{\log \left( S_i \right)}_{\mathcal{L}^\mathcal{R}_i(\mathbf{Z})},
    \vspace*{-0.3cm}
\end{equation}
where $\hat{z}$ is the unit vector of $z$, $\text{sim}(a, b) = \exp \left( \hat{a}^\top \hat{b} \right)$ is the exponent of the cosine similarity, and $S_i = \sum_{k \not\sim i} \text{sim}(z_i, z_k)$ is the sum over the repulsive terms. For clarity, we split the loss term into attraction $\mathcal{L}^\mathcal{A}_i(\mathbf{Z})$ and repulsion $\mathcal{L}^\mathcal{R}_i(\mathbf{Z})$ loss functions\footnote{The result of optimizing the attraction term is that the positive pair of embeddings is pulled closer together. Similarly, optimizing the repulsion term causes the negative embedding pairs to be pushed apart from each other.}, %We will refer to $\mathcal{L}^\mathcal{A}_i$ as the \emph{only-attractions} loss and $\mathcal{L}^\mathcal{R}_i$ as the \emph{only-repulsions} loss.
%since the gradients of each component encourage attractions and repulsions respectively among corresponding embeddings. 
where $\mathcal{L}^\mathcal{A}_i(\mathbf{Z})$ is the negative cosine similarity between $z_i$ and $z_j$. Surprisingly, \emph{BYOL}~\cite{byol} and \emph{SimSiam}~\cite{simsiam} showed that one can optimize $\mathcal{L}^\mathcal{A}_i$, the cosine similarity between positive pairs, and avoid collapse by only applying the gradients to embedding $z_i$ (rather than to both $z_i$ and $z_j$). We refer to these as \emph{non-contrastive} methods. 

\begin{figure}
    \centering
    \includegraphics[trim={5cm 4.2cm 12cm 2.7cm},clip,width=0.27\linewidth]{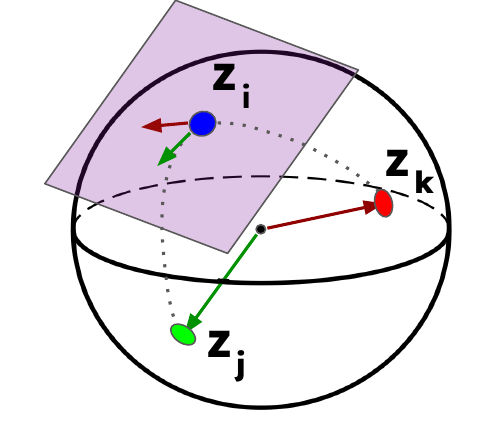}
    % \captionof{figure}{The gradients w.r.t. $z_i$ in Proposition~\ref{prop:cos_sim_grads} and Corollary~\ref{cor:grad_grows} exclusively exist in the tangent space at $\vec{z}_i$.}
    % \label{fig:grad_visualization}
    \quad \quad \quad
    \includegraphics[trim={1cm 4.5cm 8.5cm 4.5cm},clip,width=0.6\linewidth]{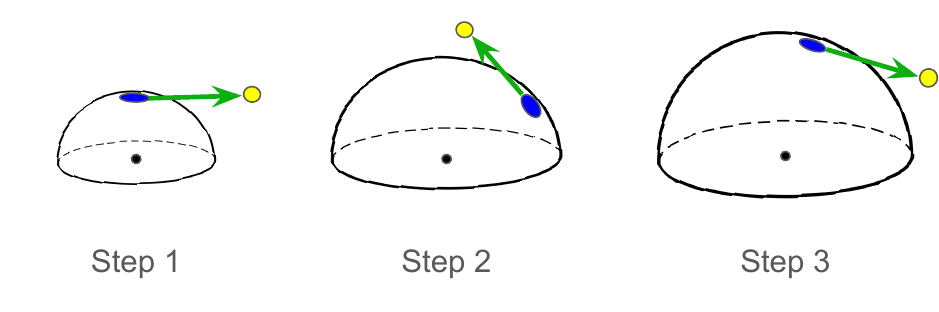}
    % \captionof{figure}{A visualization of the growing embeddings in Corollary~\ref{cor:embeddings_grow}. Blue points represent $z_i$ at iterations $t = 1, 2, 3$. Yellow points represent $z_i'$, i.e. the result of each step of gradient descent.}
    % \label{fig:growing_embeddings}
    \caption{\emph{Left}: The gradients w.r.t. $z_i$ in Proposition~\ref{prop:cos_sim_grads} and Corollary~\ref{cor:infonce_grads} exclusively exist in the tangent space at $\vec{z}_i$. \emph{Right}: The growing embeddings in Corollary~\ref{cor:embeddings_grow}. Blue points represent $z_i$ at iterations $t = 1, 2, 3$. Yellow points represent $z_i'$, i.e. the result of each step of gradient descent.}
    \label{fig:embedding_visualizations} \vspace*{-0.5cm}
\end{figure}

Across SSL methods, the quality of a learned representation is measured by training classifiers on the learned representations. If the classifier achieves high accuracy then embeddings corresponding to the same class must be near one another. As is standard, we use the $k$-nn classifier accuracy (this effectively lower-bounds the linear classifier's finetuning accuracy~\cite{dinov2}).

 %We also note that virtually every modern SSL method uses the inner product among \emph{normalized} embeddings when calculating the loss values, including but not limited to \cite{moco, simclr, simsiam, byol, protonce, nnclr}.

% \begin{figure}[t!]
%     \centering
%     \includegraphics[width=\linewidth]{Images/model_diagrams.pdf}
%     \caption{Diagrams describing the models used in this paper.}
%     \label{fig:diagrams} \vspace*{-0.5cm}
% \end{figure}

\paragraph{SSL Learning Dynamics.}
The seminal work of~\cite{understanding_contr_learn} showed that repulsion-based contrastive learning must satisfy two requirements: all positive pairs must be near one another (alignment) and all negative samples must spread evenly over the hypersphere (uniformity). They then argued that these requirements lead to classes grouping together. Expanding on this blueprint, subsequent works have sought to formalize the learning capacity of contrastive methods \cite{geom_contrastive_loss, understanding_contr_learn, latent_inversion, provable_contr_guarantees, understanding_contr_learn_2} while much of the research studying non-contrastive methods has focused on how their architectural components help to prevent collapse \cite{simsiam_avoid_collapse, dim_collapse_ssl, BYOL_orthogonality, direct_pred, ssl_dim_reduction, duality_contr_non-contr}.

Importantly, both the contrastive and non-contrastive literature assumes that an embedding's direction in latent space is its relevant parameter. The magnitude's role has, to our knowledge, only been considered in \cite{normface} and \cite{spherical_embeddings}. The former shows that replacing the softmax's inner products with the cosine similarity resolves distributional issues in the softmax function. The latter identifies that embedding norms affect the cosine similarity but resolves this by simply standardizing the embeddings to have consistent norms. Our work shows that the embedding norms will always grow when optimizing the cosine similarity. Thus, we show that standardizing the norms is insufficient -- they must also be minimized. We lastly note that it is an open question why SSL models require \ltwo-normalization \cite{direct_pred, byol, common_stability_comment} in order to converge. Our work resolves this question.

%% file: Sections/3_embedding_analysis.tex
\section{The Cosine Similarity's Dynamics}

We begin with the gradients of the cosine similarity loss with respect to an arbitrary point $z_i$:

\begin{proposition}[Prop. 3 in \cite{spherical_embeddings}; proof in \ref{prf:prop_grad_grows}]
    \label{prop:cos_sim_grads}
    Let $z_i$ and $z_j$ be two points in $\mathbb{R}^d$ and define $\mathcal{L}_i^\mathcal{A}(\mathbf{Z}) = -\hat{z}_i^\top \hat{z}_j$. Let $\phi_{ij}$ be the angle between $z_i$ and $z_j$. Then the gradient of $\mathcal{L}_i^\mathcal{A}(\mathbf{Z})$ w.r.t. $z_i$ is
    \[ \nabla_i^\mathcal{A} = \frac{1}{||z_i||} \left(\mathbf{I} - \frac{z_i z_i^\top}{||z_i||^2} \right) \frac{\mathbf{z}_j}{||z_j||} = \left( \frac{\hat{z}_j}{||z_i||} \right)_{\perp z_i} \]
    where $a_{\perp b}$ is the component of $a$ orthogonal to $b$.
    This has magnitude $||\nabla_i^\mathcal{A}|| = \frac{\sin(\phi_{ij})}{||z_i||}$.
\end{proposition}

\noindent This has an easy interpretation: $\mathbf{I} - \frac{z_i z_i^\top}{||z_i||^2}$ projects the unit vector $\hat{z}_j$ onto the subspace orthogonal to $\vec{z}_i$ and it is then inversely scaled by $||z_i||$. We visualize this in Figure~\ref{fig:embedding_visualizations} (left), where the purple plane represents the projection onto the tangent space at $z_i$. Corollary~\ref{cor:infonce_grads} in Section~\ref{app:infonce_grads} of the Appendix shows that a similar behavior holds for the InfoNCE loss function. As a consequence of these results, we see that optimizing the cosine similarity or the InfoNCE loss can \emph{only grow the embeddings}:
\begin{corollary}[Proof in~\ref{prf:cor_embeddings_grow}]
    \label{cor:embeddings_grow}
    Let $z_i$ and $z_j$ be positive embeddings with angle $\phi_{ij}$. Let $z_i'$ be the embedding after applying the gradients in Proposition~\ref{prop:cos_sim_grads} or Corollary~\ref{cor:infonce_grads} to $z_i$ via a step of gradient descent. Then
    $||z_i'|| \geq ||z_i||$.
\end{corollary}

\noindent We consider this surprising: one would expect that optimizing the cosine similarity would bring points ``in'' rather than ``out''. Nonetheless, the
results in Prop.~\ref{prop:cos_sim_grads} and Corollary~\ref{cor:embeddings_grow} reveal an inevitable catch-22: we require small embeddings to optimize the
cosine similarity but optimizing the cosine similarity grows the embeddings. This is the key dynamic that \cite{spherical_embeddings} missed. Note, these
results do not depend on the embedding pipeline -- they simply assume that we apply the cosine similarity onto some distribution $\mathbf{Z}$. Furthermore,
these results hold for the mean squared error between normalized embeddings, since $||\hat{z}_i - \hat{z}_j||_2^2 = 2 - 2\hat{z}_i^\top \hat{z}_j$. We visualize
Corollary~\ref{cor:embeddings_grow} in Figure~\ref{fig:embedding_visualizations} (right).

%% file: Sections/4_verification.tex
\paragraph{Effects on Convergence.}
Our key insights lie in the cosine similarity's (and, by extension, the InfoNCE loss's) gradient magnitude. Namely, it goes to zero when an embedding has large
norm or has a large angle to its positive-pair counterpart. We refer to these as the \emph{embedding-norm} and \emph{opposite-halves} effects, respectively. We
note that both are counter-intuitive: the former contradicts SSL intuition that an embedding's direction is the relevant term~\cite{understanding_contr_learn}
and the latter implies that, for half of the loss-landscape, the gradients shrink as the loss grows. Despite this, both quadratically slow down convergence:

\begin{theorem}[Proof in \ref{prf:thm_convergence_rate}]
    Let $z_i$ and $z_j$ be embeddings with equal norm, i.e. $||z_i|| = ||z_j|| = \rho$. Let $z_i' = z_i + \frac{\gamma}{\rho}(z_j)_{\perp z_i}$ and $z_j' = z_j + \frac{\gamma}{\rho}(z_i)_{\perp z_j}$ be the embeddings after maximizing the cosine similarity via a step of gradient descent with learning rate $\gamma$. Then the change in cosine similarity is bounded from above by:
        \begin{equation}
            \label{eq:thm_statement}
            \hat{z}_i'^\top \hat{z}_j' - \hat{z}_i^\top \hat{z}_j < \frac{2 \gamma \sin^2 \phi_{ij}}{\rho^2}.
        \end{equation}
    
    \label{thm:convergence_rate}
\end{theorem}

\noindent \begin{minipage}{0.35\textwidth}
To visualize this slowdown, we apply the cosine similarity gradients to sets of random samples in $\mathbb{R}^{20}$. By varying the mean embedding norms and
$\phi_{ij}$ values, we can then observe how much each parameter affects the convergence. Namely, we apply the gradients in Prop.~\ref{prop:cos_sim_grads} until
the mean cosine similarity exceeds $0.999$. We plot the number of steps until convergence in Figure~\ref{fig:convergence_rate}. There, we see that although

\end{minipage}
\quad
\begin{minipage}{0.6\textwidth}
    \centering
    \includegraphics[width=0.9\textwidth]{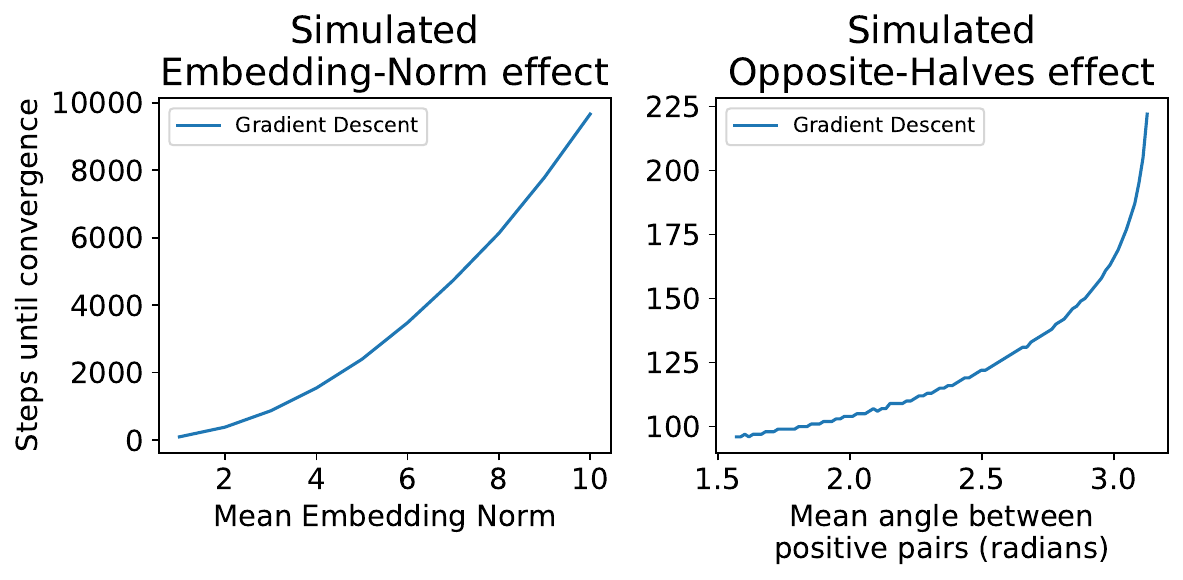}
    \captionof{figure}{The effect of the embedding norm and angle between positive samples on the convergence rate.}
    \label{fig:convergence_rate}
\end{minipage} \vspace*{0.1cm}

\noindent both effects incur quadratic slowdowns, having large embedding norms is \emph{significantly} worse for optimization ($\sim\!\!100\times$ slowdown)
than having large angles between positive pairs ($\sim\!2\times$ slowdown).

We now discuss whether these effects occur in practice. First, we plot the embedding norms during training on Cifar10 in Figure~\ref{fig:embed_mags} and see
that, across SSL models, the norms indeed grow during training. Thus, Corollary~\ref{cor:embeddings_grow} is experimentally supported. Furthermore, applying
\ltwo-normalization counteracts this growth since it forces the embeddings to shrink over time. Consequently, Table~\ref{tbl:embedding_norm_effect} shows that
\ltwo-normalization directly affects SSL model quality -- small embeddings lead to higher $k$-nn accuracies. We lastly plot the rate at which positive pairs
have angle greater than $\frac{\pi}{2}$ in Table~\ref{tbl:opposite_halves_effect} and see that, although positive pairs may have large angle at epoch 1, by
epoch 16 this effectively stops occurring. We conclude that the embedding-norm effect indeed slows down SSL training while the opposite-halves effect has
a negligible impact (further evidence in Appendix~\ref{app:opposite_halves_effect}).  \vspace*{0.25cm}

\input{Figures/embedding_mags}

    %\hspace*{-0.45cm}
    % \begin{tabular}{lccc}
    %     \multicolumn{4}{c}{\small{Embedding Magnitudes}} \\
    %     \midrule
    %     \makecell{Weight\\Decay} & SimCLR & SimSiam & BYOL \\
    %     \midrule
    %     $\lambda=$ 5e-2 & 0.0 & 0.0 & 0.0 \\
    %     $\lambda=$ 5e-3 & 0.10 & 0.47 & 0.45 \\
    %     $\lambda=$ 5e-4 & 0.41 & 0.81 & 0.70 \\
    %     $\lambda=$ 5e-6 & 13.5 & 24.5 & 23.2 \\
    %     $\lambda=$ 0 & 16.4 & 31.4 & 30.1 
    % \end{tabular}
    % \;
\hspace*{-0.7cm}\begin{minipage}{0.48\linewidth}
    \fontsize{8.5pt}{8.5pt}\selectfont
    \hspace*{-0.2cm}
    \begin{tabular}{lcccccc}
        \multirow{2}{*}{\makecell{\small Weight\\\small Decay}} & \multicolumn{2}{c}{\small SimCLR} & \multicolumn{2}{c}{\small SimSiam} & \multicolumn{2}{c}{\small BYOL} \\
        \cmidrule{2-7}
        & $k$-nn & $||z_i||$ & $k$-nn & $||z_i||$ & $k$-nn & $||z_i||$ \\
        \midrule
        $\lambda=$ 5e-$2$ & 10.0 & 0.0 & 10.0 & 0.0 & 10.0 & 0.0 \\
        $\lambda=$ 5e-3 & 54.0 & 0.10 & \underline{57.2} & 0.47 & \textbf{78.0} & 0.45 \\
        $\lambda=$ 5e-4 & \textbf{81.3} & 0.41 & \textbf{73.3} & 0.81 & \underline{74.3} & 0.70 \\
        $\lambda=$ 5e-6 & \underline{75.7} & 13.5 & 48.2 & 24.5 & 47.8 & 23.2 \\
        $\lambda=$ 0 & 75.3 & 16.4 & 47.9 & 31.4 & 47.3 & 30.1 
    \end{tabular}
    \captionof{table}{Effect of \ltwo-normalization on SSL training for Cifar10. We report the $k$-nn accuracy and the embedding norms for various weight decay parameters and SSL models.}
    \label{tbl:embedding_norm_effect}
\end{minipage}
\quad \quad
\begin{minipage}{0.45\linewidth}
    \fontsize{8.5pt}{8.5pt}\selectfont
    \begin{tabular}{lrcc}
    Model & Dataset \quad\quad & \makecell{Effect Rate\\Epoch 1} & \makecell{Effect Rate\\Epoch 16} \\
    \midrule
    \multirow{2}{*}{SimCLR} & Imagenet-100 & 2\% & 0\%  \\
    & Cifar-100 & 11\% & 1\% \\
    \cmidrule{1-4}
    \multirow{2}{*}{SimSiam} & Imagenet-100 & 26\% & 1\% \\
    & Cifar-100 & 21\% & 0\% \\
    \cmidrule{1-4}
    \multirow{2}{*}{BYOL} & Imagenet-100 & 28\% & 1\% \\
    & Cifar-100 & 20\% & 0\% \\
    \end{tabular}
    \captionof{table}{The rate at which embeddings are on opposite sides of the latent space for various datasets and SSL models.}
    \label{tbl:opposite_halves_effect}
\end{minipage}

%% file: Figures/embedding_mags.tex
\begin{figure*}[t!]
    \centering
    \includegraphics[width=0.85\textwidth]{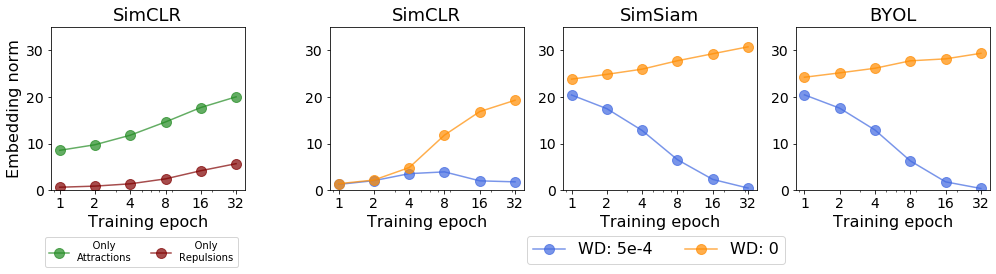}
    \caption{\emph{Left}: The embedding norms when optimizing the subsets of the InfoNCE loss function for SimCLR. \emph{Right}: The mean embedding norms for SimCLR/SimSiam/BYOL as a function of the weight-decay. Note, we use the terms \ltwo-normalization and weight-decay interchangeably.}
    \label{fig:embed_mags} \vspace*{-0.25cm}
\end{figure*}

%% file: Sections/5_experiments.tex
\section{Cut-Initialization}
\label{sec:norm_experiments}

The previous sections showed that the embedding-norm effect occurs in SSL models and that weight decay eventually remedies it. Rather than waiting for the weight-decay to slowly shrink the embeddings, we propose to initialize the embeddings directly at the correct size via \emph{cut-initialization}: an initialization scheme where we divide each layer's weights with cut-constant $c>1$ before training. 

We refer to Tables~\ref{tbl:simclr_wd_cut} and \ref{tbl:simsiam_wd_cut} to study the relationship between the cut-constant and the weight-decay. Specifically, we analyze the classification accuracy after 100 epochs and see that, in both the contrastive and non-contrastive settings, pairing cut-initialization with weight decay accelerates the training process. Based on these results, we default to cut-constant $c=3$ for contrastive methods and $c=9$ for non-contrastive ones. Table~\ref{tbl:accuracies} then evaluates the effectiveness of cut-initialization on full training runs\footnote{Due to lack of compute, we use a batch-size of 128 with the Resnet50 backbone, indicated by the asterisk.}. The takeaway is that cut-initialization essentially guarantees faster convergence.

\begin{table}[t!]
    \centering
    \fontsize{7pt}{7pt}\selectfont
        \begin{tabular}{cr cc ccc ccc ccc}
        \toprule
        & & \multicolumn{2}{c}{Cifar10} & \multicolumn{3}{c}{Cifar100} & \multicolumn{3}{c}{Imagenet-100} & \multicolumn{3}{c}{\makecell{Modified\\Flowers}} \\
        \cmidrule(l{2pt}r{2pt}){3-4} \cmidrule(l{2pt}r{2pt}){5-7} \cmidrule(l{2pt}r{2pt}){8-10} \cmidrule(l{2pt}r{2pt}){11-13}
        \vspace*{0.05cm}
        & & \multicolumn{2}{c}{Epoch} & \multicolumn{3}{c}{Epoch} & \multicolumn{3}{c}{Epoch} & \multicolumn{3}{c}{Epoch} \\ 
        & & \makecell{100} & \makecell{500} & \makecell{100} & \makecell{300} & \makecell{500} & \makecell{100} & \makecell{300} & \makecell{500} & \makecell{100} & \makecell{300} & \makecell{500} \\
        \cmidrule(l{2pt}r{2pt}){1-2} \cmidrule(l{2pt}r{2pt}){3-4} \cmidrule(l{2pt}r{2pt}){5-7} \cmidrule(l{2pt}r{2pt}){8-10} \cmidrule(l{2pt}r{2pt}){11-13}
        \multirow{2}{*}{SimCLR$^*$} & Default & 75.7 & 85.2 & 33.9 & 48.1 & 51.9 & 44.6 & 55.9 & 59.4 & 18.9 & 28.3 & 35.8 \\ % cifar100 33.2
        & Cut ($c=3$) & 79.6 & 86.1 & 42.3 & 50.8 & 52.6 & 46.7 & 58.1 & 60.9 & 23.7 & 50.8 & 65.3 \\ % cifar100 42.3 & 
        \cmidrule(l{2pt}r{2pt}){1-2} \cmidrule(l{2pt}r{2pt}){3-4} \cmidrule(l{2pt}r{2pt}){5-7} \cmidrule(l{2pt}r{2pt}){8-10} \cmidrule(l{2pt}r{2pt}){11-13}
        \multirow{2}{*}{SimSiam} & Default & 67.1 & 87.0 & 23.7 & 50.1 & 58.6 & 33.6 & 61.1 & 62.0 & 15.3 & 16.5 & 16.9 \\
        & Cut ($c=9$) & 77.7 & 89.0 & 41.1 & 59.4 & 61.7 & 52.9 & 63.7 & 67.2 & 15.9 & 27.6 & 38.5 \\
        \cmidrule(l{2pt}r{2pt}){1-2} \cmidrule(l{2pt}r{2pt}){3-4} \cmidrule(l{2pt}r{2pt}){5-7} \cmidrule(l{2pt}r{2pt}){8-10} \cmidrule(l{2pt}r{2pt}){11-13}
        \multirow{2}{*}{BYOL} & Default & 69.5 & 88.2 & 18.8 & 53.5 & 61.9 & 26.0 & 65.8 & 67.7 & 16.1 & 17.4 & 17.1 \\
        & Cut ($c=9$) & 82.2 & 88.6 & 46.2 & 62.4 & 61.4 & 54.2 & 68.7 & 68.6 & 15.9 & 27.6 & 43.1 \\
        \bottomrule
        \end{tabular}
    \caption{$k$-nn accuracies for default and cut-initialized training.}\vspace*{-.7cm}
    \label{tbl:accuracies}
\end{table}
% simsiam 300 flowers: 16.5 and 27.2
% byol 300 flowers: 17.4 and 27.6

We note that the original Flowers102 dataset \cite{flowers} is nearly impossible for self-supervised image recognition. We therefore create a new training split using the supplementary imbalanced class examples. Each class in our \emph{modified-flowers} dataset has between 50 and 250 samples, implying a long-tailed data distribution~\cite{long-tail-survey}. Despite this still being challenging for the default models, incorporating cut-initialization allows them to obtain significantly higher accuracies in Table~\ref{tbl:accuracies}. Thus, not only does cut-initialization accelerate training, it can sometimes facilitate it entirely.

It remains to show that cut-initialization works across SSL models and backbone architectures. To this end, Table~\ref{tbl:accuracies} shows that cut-initialization still provides the expected boost in performance for MoCov2 and for MoCov3 (which uses a ViT-s trasnformer backbone). We point out that Dino's objective function does not depend on the cosine similarity~\cite{dino}. Consequently, it does not (and should not) benefit from cut-initialization. This supports our analysis: using cut-initialization in theoretically unmotivated settings no longer has a positive impact on training. \vspace*{0.3cm}

\hspace*{-1.3cm}
\noindent\begin{minipage}{0.28\textwidth}
    \centering
    \fontsize{8pt}{8pt}\selectfont
    \def\arraystretch{1.165}
    \vspace*{-0.22cm}
    \begin{tabular}{lccc}
        \midrule
         % & & \multicolumn{2}{c}{$k$-nn Accuracy} \\
         & & Top-1 & Top-5 \\
         \midrule
         \multirow{2}{*}{\rotatebox[origin=c]{90}{\makecell{\fontsize{6.7pt}{6.7pt}\selectfont MoCo\\V2}}} & $c=1$ & 38.7 & 69.7 \\
         & $c=3$ & 43.8 & 71.8 \\
         \midrule
         \multirow{2}{*}{\rotatebox[origin=c]{90}{\makecell{\fontsize{6.7pt}{6.7pt}\selectfont MoCo\\V3}}} & $c=1$ & 54.8 & 82.7 \\
         & $c=3$ & 58.8 & 86.0 \\
         \midrule
         \multirow{2}{*}{\rotatebox[origin=c]{90}{\makecell{\fontsize{6.7pt}{6.7pt}\selectfont DINO}}} & $c=1$ & 43.7 & 76.6 \\
         & $c=3$ & 29.0 & 55.6 \\
         \bottomrule
    \end{tabular}
    \captionof{table}{Imagenet-100 $k$-nn accuracy (epoch 100) for MoCo and Dino.}
    \label{tbl:transformer_accs}
\end{minipage}
\;
\begin{minipage}{0.39\linewidth}
    \vspace*{-0.04cm}
    \fontsize{8.5pt}{8.5pt}\selectfont
    \centering
    \begin{tabular}{cl cccc}
    \multicolumn{6}{c}{\fontsize{10pt}{10pt}\selectfont SimCLR} \\
    \midrule \vspace*{0.1cm}
    & & \multicolumn{4}{c}{Weight Decay} \\
    & & 1e-8 & \underline{1e-6} & 5e-6 & 1e-5 \\
    \cmidrule{3-6}
    \multirow{4}{*}{\rotatebox[origin=c]{90}{\makecell{Cut}}} & $c=1$ & 40.8 & 40.5 & 40.9 & 41.5 \\
    & $c=2$ & 42.7 & 42.8 & 42.9 & 42.2 \\
    & $c=4$ & 42.3 & 41.4 & 42.0 & 41.1 \\
    & $c=8$ & 37.1 & 36.8 & 37.9 & 37.3 \\
    \bottomrule
    \end{tabular}
    \captionof{table}{SimCLR $k$-nn accuracy (epoch 100) on Cifar100 for various cut-constants and weight-decays. Default weight-decay is underlined.}
    \label{tbl:simclr_wd_cut}
\end{minipage}
\;
\begin{minipage}{0.38\linewidth}
    \fontsize{8.5pt}{8.5pt}\selectfont
    %\vspace*{0.38cm}
    \centering
    \begin{tabular}{cl cccc}
    \multicolumn{6}{c}{\fontsize{10pt}{10pt}\selectfont SimSiam} \\
    \midrule \vspace*{0.1cm}
    & & \multicolumn{4}{c}{Weight Decay} \\
    & & 5e-5 & 1e-4 & \underline{5e-4} & 1e-3 \\
    \cmidrule{3-6}
    \multirow{4}{*}{\rotatebox[origin=c]{90}{\makecell{Cut}}} & $c=1$ & 36.7 & 38.5 & 40.5 & 44.7 \\
    & $c=2$ & 41.1 & 44.0 & 48.8 & 42.1 \\
    & $c=4$ & 40.2 & 41.2 & 49.8 & 49.1 \\
    & $c=8$ & 44.4 & 46.4 & 50.2 & 53.2 \\
    \bottomrule
    \end{tabular}
    \captionof{table}{SimSiam $k$-nn accuracy (epoch 100) on Imagenet-100 for various values of $c$ and $\lambda$. Default weight-decay is underlined.}
    \label{tbl:simsiam_wd_cut}
\end{minipage}

%% file: Sections/9_appendix.tex
\section{Proofs}
\subsection{Proof of Proposition~\ref{prop:cos_sim_grads}}
\label{prf:prop_grad_grows}
\begin{proof}
    We are taking the gradient of $\mathcal{L}^\mathcal{A}_i$ as a function of $z_i$. The principal idea is that the gradient has a term with direction $\hat{z}_j$ and a term with direction $-\hat{z}_i$. We then disassemble the vector with direction $\hat{z}_j$ into its component parallel to $z_i$ and its component orthogonal to $z_i$. In doing so, we find that the two terms with direction $z_i$ cancel, leaving only the one with direction orthogonal to $z_i$.
    
    Writing it out fully, we have $\mathcal{L}^\mathcal{A}_i = z_i^\top z_j / (||z_i|| \cdot ||z_j||)$. Taking the gradient of this then amounts to using the quotient rule, with $f = z_i^\top z_j$ and $g = ||z_i|| \cdot ||z_j|| = \sqrt{z_i^\top z_i} \cdot \sqrt{z_j^\top z_j}$. Taking the derivative of each, we have
    \begin{align*}
        f' &= \mathbf{z}_j \\
        g' &= ||z_j|| \frac{z_i}{\sqrt{z_i^\top z_i}} = ||z_j|| \frac{\mathbf{z}_i}{||z_i||} \\
        \implies \frac{f' g - g' f}{g^2} &= \frac{\left(\mathbf{z}_j \cdot ||z_i|| \cdot ||z_j|| \right) - \left(||z_j|| \frac{\mathbf{z}_i}{||z_i||} \cdot z_i^\top z_j \right)}{||z_i||^2 \cdot ||z_j||^2} \\
        &= \frac{\mathbf{z}_j}{||z_i|| \cdot ||z_j||} - \frac{\mathbf{z}_i z_i^\top z_j}{||z_i||^3 ||z_j||},
    \end{align*}
    where we use boldface $\mathbf{z}$ to emphasize which direction each term acts along. We now substitute $\cos(\phi_{ij}) = z_i^\top z_j / (||z_i|| \cdot ||z_j||)$ in the second term to get
    \begin{equation}
        \label{eq:quotient_rule}
        \frac{f' g - g' f}{g^2} = \frac{\hat{z}_j}{||z_i||} - \frac{\mathbf{z}_i \cos(\phi)}{||z_i||^2}
    \end{equation}

    It remains to separate the first term into its sine and cosine components and perform the resulting cancellations. To do this, we take the projection of $\hat{z}_j = \mathbf{z}_j / ||z_j||$ onto $\mathbf{z}_i$ and onto the plane orthogonal to $\mathbf{z}_i$. The projection of $\hat{z}_j$ onto $\mathbf{z}_i$ is given by
    \[ \cos \phi_{ij} \frac{\mathbf{z}_i}{||z_i||} \]
    while the projection of $\mathbf{z}_j / ||z_j||$ onto the plane orthogonal to $\mathbf{z}_i$ is
    \[ \left( \mathbf{I} - \frac{z_i z_i^\top}{||z_i||^2} \right) \frac{\mathbf{z}_j}{||z_j||}. \]
    It is easy to assert that these components sum to $\mathbf{z}_j/||z_j||$ by replacing the $\cos \phi_{ij}$ by $\frac{z_i^\top z_j}{||z_i||\cdot ||z_j||}$.

    We plug these into Eq.~\ref{eq:quotient_rule} and cancel the first and third term to arrive at the desired value:
    \begin{align*}
        \frac{f' g - g' f}{g^2} = &\frac{1}{||z_i||} \cos \phi \frac{\mathbf{z}_i}{||z_i||} \\
        &+ \frac{1}{||z_i||} \cdot \left( \mathbf{I} - \frac{z_i z_i^\top}{||z_i||^2} \right) \frac{\mathbf{z}_j}{||z_j||} \\
        &- \frac{\mathbf{z}_i \cos(\phi)}{||z_i||^2} \\
        = &\frac{1}{||z_i||} \cdot \left( \mathbf{I} - \frac{z_i z_i^\top}{||z_i||^2} \right) \frac{\mathbf{z}_j}{||z_j||}.
    \end{align*}
\end{proof}

\subsection{Proof of Corollary~\ref{cor:embeddings_grow}}
\label{prf:cor_embeddings_grow}
\begin{proof}
    First, consider that we applied the cosine similarity's gradients from Proposition~\ref{prop:cos_sim_grads}. Since $z_i$ and $(z_j)_{\perp z_i}$ are orthogonal, $||z_i'||_2^2 = ||z_i||^2 + \frac{\gamma^2}{||z_i||^2}||(z_j)_{\perp z_i}||^2$. The second term is positive if $\sin \phi_{ij} > 0$. Figure~\ref{fig:embedding_visualizations} visualizes this.

    The same exact argument holds for the InfoNCE gradients. The gradient is orthogonal to the embedding, so a step of gradient descent can only increase the embedding's magnitude.
\end{proof}

\subsection{Proof of Theorem~\ref{thm:convergence_rate}}
\label{prf:thm_convergence_rate}
We first restate the theorem:

Let $z_i$ and $z_j$ be positive embeddings with equal norm, i.e. $||z_i|| = ||z_j|| = \rho$. Let $z_i'$ and $z_j'$ be the embeddings after 1 step of gradient descent with learning rate $\gamma$. Then the change in cosine similarity is bounded from above by:
\begin{equation*}
    \hat{z}_i'^\top \hat{z}_j' - \hat{z}_i^\top \hat{z}_j < \frac{\gamma \sin^2 \phi_{ij}}{\rho^2} \left[ 2 - \frac{\gamma \cos \phi}{\rho^2} \right].
\end{equation*}

\noindent We now proceed to the proof:
\begin{proof}
    Let $z_i$ and $z_j$ be two embeddings with equal norm\footnote{We assume the Euclidean distance for all calculations.}, i.e. $||z_i|| = ||z_j|| = \rho$. We then perform a step of gradient descent to maximize $\hat{z}_i^\top \hat{z}_j$. That is, using the gradients in \ref{prop:cos_sim_grads} and learning rate $\gamma$, we obtain new embeddings $z_i' = z_i + \frac{\gamma}{||z_i||} (\hat{z}_j)_{\perp z_i}$ and $z_j' = z_j + \frac{\gamma}{||z_j||} (\hat{z}_i)_{\perp z_j}$. Going forward, we write $\delta_{ij} = (\hat{z}_j)_{\perp z_i}$ and $\delta_{ji} = (\hat{z}_i)_{\perp z_j}$, so $z_i' = z_i + \frac{\gamma}{\rho} \delta_{ij}$ and $z_j' = z_j + \frac{\gamma}{\rho} \delta_{ji}$. Notice that since $z_i$ and $\delta_{ij}$ are orthogonal, by the Pythagorean theorem we have $||z_i'||^2 = ||z_i||^2 + \frac{\gamma^2}{\rho^2}||\delta_{ij}||^2 \geq ||z_i||^2$. Lastly, we define $\rho' = ||z_i'|| = ||z_j'||$.

    We are interested in analyzing $\hat{z}_i'^\top \hat{z}_j' - \hat{z}_i^\top \hat{z}_j$. To this end, we begin by re-framing $\hat{z}_i'^\top \hat{z}_j'$:
    \begin{align*}
        \hat{z}_i'^\top \hat{z}_j' &= \left(\frac{z_i + \frac{\gamma}{\rho} \delta_{ij}}{\rho'}\right)^\top \left(\frac{z_j + \frac{\gamma}{\rho} \delta_{ji}}{\rho'}\right) \\
        &= \frac{1}{\rho'^2}\left[ z_i^\top z_j + \gamma \frac{z_i^\top \delta_{ji}}{\rho'} + \gamma \frac{z_j^\top \delta_{ij}}{\rho'} + \gamma^2 \frac{\delta_{ij}^\top \delta_{ji}}{\rho'^2} \right].
    \end{align*}

    We now consider that, since $\delta_{ij}$ is the projection of $\hat{z}_j$ onto the subspace orthogonal to $z_i$, we have that the angle between $z_i$ and $\delta_{ji}$ is $\pi/2 - \phi_{ij}$. Plugging this in and simplifying, we obtain
    \begin{align*}
        z_i^\top \delta_{ji} &= ||z_i|| \cdot ||\delta_{ji}|| \cos (\pi/2 - \phi_{ij}) \\
        &= ||z_i|| \cdot ||\delta_{ji}|| \sin \phi_{ij} \\
        &= \rho \sin^2 \phi_{ij}.
    \end{align*}
    By symmetry, the same must hold for $z_j^\top \delta_{ij}$.
    
    Similarly, we notice that the angle $\psi_{ij}$ between $\delta_{ij}$ and $\delta_{ji}$ is $\psi_{ij} = \pi - \phi_{ij}$. The reason for this is that we must have a quadrilateral whose four internal angles must sum to $2\pi$, i.e. $\psi_{ij} + \phi_{ij} + 2 \frac{\pi}{2} = 2 \pi$. Thus, we obtain $\delta_{ij}^\top \delta_{ji} = ||\delta_{ij}|| \cdot ||\delta_{ji}|| \cos(\psi) = -\sin^2 \phi_{ij} \cos \phi_{ij}$.

    We plug these back into our equation for $\hat{z}_i'^\top \hat{z}_j'$ and simplify:
    \begin{align*}
        \hat{z}_i'^\top \hat{z}_j' &= \frac{1}{\rho'^2}\left[ z_i^\top z_j + \gamma \frac{z_i^\top \delta_{ji}}{\rho} + \gamma \frac{z_j^\top \delta_{ij}}{\rho} + \gamma^2 \frac{\delta_{ij}^\top \delta_{ji}}{\rho^2} \right] \\
        &= \frac{1}{\rho'^2}\left[ z_i^\top z_j + \gamma \frac{\rho \sin^2 \phi_{ij}}{\rho} + \gamma \frac{\rho \sin^2 \phi_{ij}}{\rho} - \gamma^2 \frac{\sin^2 \phi_{ij} \cos \phi_{ij}}{\rho^2} \right] \\
        &= \frac{1}{\rho'^2}\left[ z_i^\top z_j + 2 \gamma \sin^2 \phi_{ij} - \gamma^2 \frac{\sin^2 \phi_{ij} \cos \phi_{ij}}{\rho^2} \right].
    \end{align*}

    We now consider the original term in question:
    \begin{align*}
        \hat{z}_i'^\top \hat{z}_j' - \hat{z}_i^\top \hat{z}_j &= \frac{1}{\rho'^2}\left[ z_i^\top z_j + 2 \gamma \sin^2 \phi_{ij} - \gamma^2 \frac{\sin^2 \phi_{ij} \cos \phi_{ij}}{\rho^2} \right] - \frac{z_i^\top z_j}{\rho^2} \\
        &\leq \frac{1}{\rho^2}\left[ z_i^\top z_j + 2 \gamma \sin^2 \phi_{ij} - \gamma^2 \frac{\sin^2 \phi_{ij} \cos \phi_{ij}}{\rho^2} \right] - \frac{z_i^\top z_j}{\rho^2} \\
        &= \frac{1}{\rho^2}\left[ 2 \gamma \sin^2 \phi_{ij} - \gamma^2 \frac{\sin^2 \phi_{ij} \cos \phi_{ij}}{\rho^2} \right] \\
        &= \frac{\gamma \sin^2 \phi_{ij}}{\rho^2}\left[ 2 - \frac{\gamma \cos \phi_{ij}}{\rho^2} \right]\\
        &\leq \frac{2 \gamma \sin^2 \phi_{ij}}{\rho^2}
    \end{align*}
    
    This concludes the proof.
\end{proof}

\section{InfoNCE Gradients}
\label{app:infonce_grads}

The InfoNCE gradient has a similar structure to the cosine similarity's:
\begin{corollary}
    \label{cor:infonce_grads}
    Let $\mathbf{Z}$ be a set of $b$ embeddings in $\mathbb{R}^d$, with $z_i$ corresponding to the $i$-th row of $\mathbf{Z}$. Let $\mathcal{L}_i(\mathbf{Z}) = \mathcal{L}_i^\mathcal{A}(\mathbf{Z}) + \mathcal{L}_i^\mathcal{R}(\mathbf{Z})$. Then the gradient of $\mathcal{L}_i(\mathbf{Z})$ w.r.t. $z_i$ is
    \begin{equation}
        \label{eq:infonce_grads}
        \nabla^{InfoNCE}_i = \frac{1}{||z_i||} \left[ \left( \hat{z}_j \right)_{\perp z_i} + \sum_{k \not\sim i} \left( \hat{z}_k \cdot \frac{\text{\emph{sim}}(z_i, z_k)}{S_i} \right)_{\perp z_i} \right].
    \end{equation}\vspace*{-0.2cm}
    \label{cor:infonce_grads}
\end{corollary}
\begin{proof}
    We are interested in the gradient of $\mathcal{L}_i^\mathcal{R}$ with respect to $z_i$. By the chain rule, we get
    \begin{align*}
        \nabla_i^\mathcal{R} &= \frac{\sum_{k \not\sim i} \text{sim}(z_i, z_k) \frac{\partial \frac{z_i^\top z_k}{||z_i|| \cdot ||z_k||}}{\partial z_i}}{\sum_{k \not\sim i} \text{sim}(z_i, z_k)} \\
        &= \frac{\sum_{k \not\sim i} \text{sim}(z_i, z_k) \frac{\partial \frac{z_i^\top z_k}{||z_i|| \cdot ||z_k||}}{\partial z_i}}{S_i}
    \end{align*}
    It remains to substitute the result of Prop. \ref{prop:cos_sim_grads} for $\partial \frac{z_i^\top z_k}{||z_i|| \cdot ||z_k||} / \partial z_i$.

    We sum this this with the gradients of the attractive term to obtain the full InfoNCE gradient, completing the proof.
\end{proof}

In essence, because the InfoNCE loss is a function of the cosine similarity, the chain rule implies that its gradients have similar properties to the cosine similarity's. Specifically, just like those of $\mathcal{L}_i^\mathcal{A}$, the gradients of $\mathcal{L}_i^\mathcal{R}$ have the properties that (1) they are inversely scaled by $||z_i||$, (2) they are only defined in $z_i$'s tangent space, and (3) optimizing the repulsions cause the embeddings to grow\footnote{The primary difference is that the repulsive gradients have a weighted average over all the negative samples $z_k$. We note that this weighted average over a set of unit vectors leads to the repulsive force being smaller than the attractive one.}. Since the InfoNCE loss is the sum of $\mathcal{L}_i^\mathcal{A}$ and $\mathcal{L}_i^\mathcal{R}$, these properties all extend to the InfoNCE loss as well. 

\section{Further Discussion and Experiments}
\label{app:experiments}

\subsection{Opposite-Halves Effects}
\label{app:opposite_halves_effect}

We first consider how much the opposite-halves effect plays a role in SSL training in theory. Referring back to Figure~\ref{fig:convergence_rate}, we see that the effect is most impactful when the angle between positive embeddings is close to $\pi$, i.e. $\phi_{ij} > \pi - \varepsilon$ for $\varepsilon \rightarrow 0$. The following result shows that this is exceedingly unlikely in high-dimensions:
\begin{proposition}
    \label{prop:unlikely_opp_halves}
    Let $x_i, x_j \sim \mathcal{N}(0, \mathbf{I})$ be $d$-dimensional, i.i.d. random variables and let $0 < \varepsilon < 1$. Then \vspace*{-0.1cm}
    \begin{equation}
    \label{eq:opp_halves_unlikely}
    \mathbb{P}\left[ \hat{x}_i^\top \hat{x}_j \geq 1 - \varepsilon \right] \leq \frac{1}{2d(1-\varepsilon)^2}.
    \end{equation}\vspace*{-0.3cm}
\end{proposition}
\begin{proof}
By \cite{distribution_of_cosine_sim}, the cosine similarity between two i.i.d. random variables drawn from $\mathcal{N}(0, \mathbf{I})$ has expected value $\mu = 0$ and variance $\sigma^2 = 1/d$, where $d$ is the dimensionality of the space. We therefore plug these into Chebyshev's inequality:
\begin{align*}
    &\text{Pr} \left[ \abs*{\frac{x_i^\top x_j}{||x_i||\cdot ||x_j||} - \mu} \geq k \sigma \right] \leq \frac{1}{k^2} \\
    \rightarrow & \text{Pr} \left[ \abs*{\frac{x_i^\top x_j}{||x_i||\cdot ||x_j||}} \geq \frac{k}{\sqrt{d}} \right] \leq \frac{1}{k^2}
\end{align*}

\noindent We now choose $k = \sqrt{d}(1 - \varepsilon)$, giving us
\[ \mathbb{P}\left[ \abs*{\frac{x_i^\top x_j}{||x_i|| \cdot ||x_j||}} \geq 1 - \varepsilon \right] \leq \frac{1}{d(1-\varepsilon)^2}.\]

It remains to remove the absolute values around the cosine similarity. Since the cosine similarity is symmetric around $0$, the likelihood that its absolute value exceeds $1 - \varepsilon$ is twice the likelihood that its value exceeds $1- \varepsilon$, concluding the proof.

We note that this is actually an extremely optimistic bound since we have not taken into account the fact that the maximum of the cosine similarity is 1.
\end{proof}

Thus, the opposite-halves effect is exceedingly unlikely to occur in theory.

Nonetheless, Table~\ref{tbl:opposite_halves_effect} shows that embeddings have angle greater than $\pi/2$ at a rate of $5\%$ and $25\%$ for SimCLR and SimSiam/BYOL, respectively. So even if the `strongest' variant of the opposite-halves effect is not occurring, a weaker one may still be. To visualize why this happens more for non-contrastive embeddings, we train SimCLR and BYOL with a $2$-dimensional projection/prediction latent space (as is done for MoCo in~\cite{understanding_contr_learn}) and draw lines between positive samples in Figure~\ref{fig:2d_latent_spaces}. SimCLR performs as expected -- it distributes the embeddings across the latent space and keeps positive samples close together. However, Figure~\ref{fig:2d_latent_spaces} suggests that BYOL begins by concentrating its embeddings at antipodal points on the hypersphere and does not spread out evenly from there. Drawing from \cite{byol_batchnorm, byol_without_batch_statistics, simsiam}, we attribute this to the batch-normalization that is present in BYOL and SimSiam's projection/prediction heads\footnote{Consider what happens if the distribution has collapsed to a single point right before the batch-normalization layer. Since batch-normalization subtracts the mean and divides by the variance, the single collapsed cluster becomes chaotically distributed across the latent space~\cite{byol_batchnorm}. Thus, batch-normalized models cannot collapse to a single point.}.

\begin{figure}
    \includegraphics[width=0.95\linewidth]{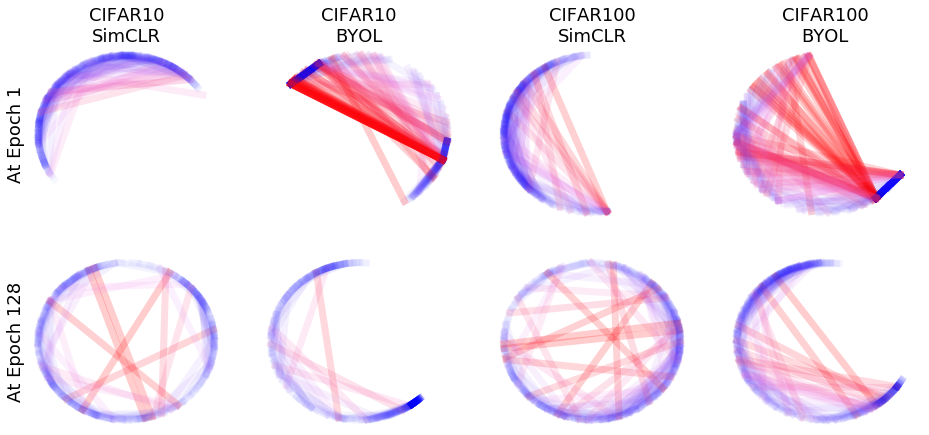}
    \caption{Lines between positive samples in 2D latent space during training. The color goes from red to blue as the cos. sim. goes from $-1$ to $1$.}
    \label{fig:2d_latent_spaces}
\end{figure}

Putting these pieces together explains why SimSiam and BYOL have a higher rate of the opposite-halves effect at the start of training. Since they only optimize attractions, they must begin by concentrating embeddings together. Due to batch-normalization, this concentration occurs in several clusters that straddle the origin. Thus, some positive pairs will necessarily have an angle greater than $\frac{\pi}{2}$ at the beginning of training (seen in the `Epoch 1' row of Figure~\ref{fig:2d_latent_spaces}).

We point out, however, that very early into training (epoch 16), every method has a rate of effectively 0 for the opposite-halves effect. This is visualized in Figure~\ref{fig:2d_latent_spaces}, where the learned representations at epoch 128 have very few `red' lines\footnote{Interestingly, Figure~\ref{fig:2d_latent_spaces} also implies that attraction-only methods never reach uniformity over the latent space.}. Furthermore, the rates in Table~\ref{tbl:opposite_half_effect} measure how often $\phi_{ij} > \frac{\pi}{2}$: the absolute weakest version of the opposite-halves effect. If we instead count the number of times that $\phi_{ij} > \frac{7\pi}{8}$ (the range where the opposite halves effect could start mattering according to Figure~\ref{fig:convergence_rate}), then it occurs 0 times across all the experiments in Table~\ref{tbl:opposite_half_effect}. Thus, it is clear that while some weak variant of the opposite-halves effect must occur at the beginning of training, it does not have a strong impact on the convergence dynamics and, in either case, disappears quite quickly.

\subsection{Cut-Initialization}
To get a preliminary intuition for cut-initialization's impact on the convergence, we plot the effect of the cut constant on the embedding norms and accuracies over training in Figure~\ref{fig:cut_experiments}\footnote{To make the effect more apparent, we use weight-decay $\lambda=5e-4$ in all models.}. We see that dividing the network's weights by $c>1$ leads to immediate convergence improvements in all models. Furthermore, this effect degrades gracefully: as $c > 1$ becomes $c < 1$, the embeddings stay large for longer and, as a result, the convergence is slower. We also see that cut-initialization has a more pronounced effect in attraction-only models -- a trend that remains consistent throughout the experiments.

\begin{figure}[t!]
    \centering
    \includegraphics[width=0.85\textwidth]{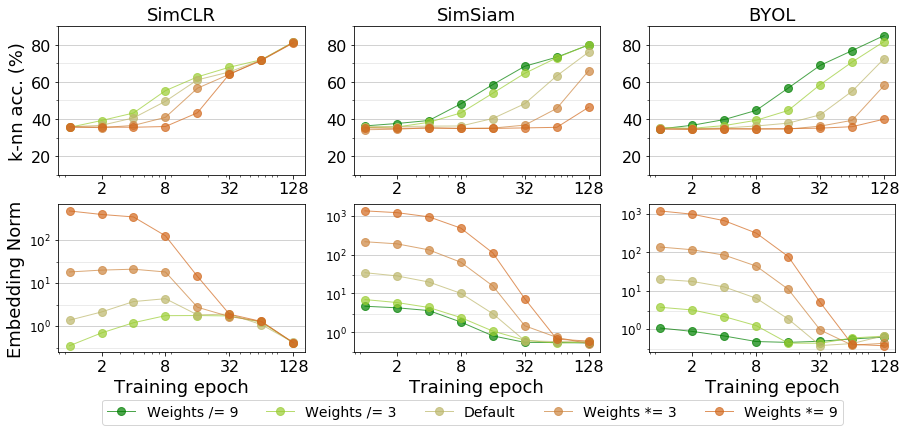}
    \caption{The effect of cut-initialization on Cifar10 SSL representations. $x$-axis and embedding norm's $y$-axis are log-scale. $\lambda=5$e$-4$ in all experiments.}
    \label{fig:cut_experiments} \vspace*{-0.7cm}
\end{figure}

%% file: main.bbl
\begin{thebibliography}{29}
\providecommand{\natexlab}[1]{#1}
\providecommand{\url}[1]{\texttt{#1}}
\expandafter\ifx\csname urlstyle\endcsname\relax
  \providecommand{\doi}[1]{doi: #1}\else
  \providecommand{\doi}{doi: \begingroup \urlstyle{rm}\Url}\fi

\bibitem[Balestriero and LeCun(2022)]{ssl_dim_reduction}
Randall Balestriero and Yann LeCun.
\newblock Contrastive and non-contrastive self-supervised learning recover
  global and local spectral embedding methods.
\newblock \emph{Advances in Neural Information Processing Systems},
  35:\penalty0 26671--26685, 2022.

\bibitem[Caron et~al.(2021)Caron, Touvron, Misra, J{\'e}gou, Mairal,
  Bojanowski, and Joulin]{dino}
Mathilde Caron, Hugo Touvron, Ishan Misra, Herv{\'e} J{\'e}gou, Julien Mairal,
  Piotr Bojanowski, and Armand Joulin.
\newblock Emerging properties in self-supervised vision transformers.
\newblock In \emph{Proceedings of the IEEE/CVF international conference on
  computer vision}, pages 9650--9660, 2021.

\bibitem[Chen et~al.(2020{\natexlab{a}})Chen, Kornblith, Norouzi, and
  Hinton]{simclr}
Ting Chen, Simon Kornblith, Mohammad Norouzi, and Geoffrey Hinton.
\newblock A simple framework for contrastive learning of visual
  representations.
\newblock In \emph{International conference on machine learning}, pages
  1597--1607. PMLR, 2020{\natexlab{a}}.

\bibitem[Chen and He(2021)]{simsiam}
Xinlei Chen and Kaiming He.
\newblock Exploring simple siamese representation learning.
\newblock In \emph{Proceedings of the IEEE/CVF conference on computer vision
  and pattern recognition}, pages 15750--15758, 2021.

\bibitem[Chen et~al.(2020{\natexlab{b}})Chen, Fan, Girshick, and He]{mocov2}
Xinlei Chen, Haoqi Fan, Ross Girshick, and Kaiming He.
\newblock Improved baselines with momentum contrastive learning.
\newblock \emph{arXiv preprint arXiv:2003.04297}, 2020{\natexlab{b}}.

\bibitem[Fetterman and Albrecht(2020)]{byol_batchnorm}
Abe Fetterman and Josh Albrecht.
\newblock Understanding self-supervised and contrastive learning with bootstrap
  your own latent (byol).
\newblock \emph{Untitled AI, August}, 2020.

\bibitem[Garrido et~al.(2022)Garrido, Chen, Bardes, Najman, and
  Lecun]{duality_contr_non-contr}
Quentin Garrido, Yubei Chen, Adrien Bardes, Laurent Najman, and Yann Lecun.
\newblock On the duality between contrastive and non-contrastive
  self-supervised learning.
\newblock \emph{arXiv preprint arXiv:2206.02574}, 2022.

\bibitem[Grill et~al.(2020)Grill, Strub, Altch{\'e}, Tallec, Richemond,
  Buchatskaya, Doersch, Avila~Pires, Guo, Gheshlaghi~Azar, et~al.]{byol}
Jean-Bastien Grill, Florian Strub, Florent Altch{\'e}, Corentin Tallec, Pierre
  Richemond, Elena Buchatskaya, Carl Doersch, Bernardo Avila~Pires, Zhaohan
  Guo, Mohammad Gheshlaghi~Azar, et~al.
\newblock Bootstrap your own latent-a new approach to self-supervised learning.
\newblock \emph{Advances in neural information processing systems},
  33:\penalty0 21271--21284, 2020.

\bibitem[HaoChen et~al.(2021)HaoChen, Wei, Gaidon, and
  Ma]{provable_contr_guarantees}
Jeff~Z HaoChen, Colin Wei, Adrien Gaidon, and Tengyu Ma.
\newblock Provable guarantees for self-supervised deep learning with spectral
  contrastive loss.
\newblock \emph{Advances in Neural Information Processing Systems},
  34:\penalty0 5000--5011, 2021.

\bibitem[He et~al.(2020)He, Fan, Wu, Xie, and Girshick]{moco}
Kaiming He, Haoqi Fan, Yuxin Wu, Saining Xie, and Ross Girshick.
\newblock Momentum contrast for unsupervised visual representation learning.
\newblock In \emph{Proceedings of the IEEE/CVF conference on computer vision
  and pattern recognition}, pages 9729--9738, 2020.

\bibitem[Henaff(2020)]{data_efficient_cpc}
Olivier Henaff.
\newblock Data-efficient image recognition with contrastive predictive coding.
\newblock In \emph{International conference on machine learning}, pages
  4182--4192. PMLR, 2020.

\bibitem[Jing et~al.(2021)Jing, Vincent, LeCun, and Tian]{dim_collapse_ssl}
Li~Jing, Pascal Vincent, Yann LeCun, and Yuandong Tian.
\newblock Understanding dimensional collapse in contrastive self-supervised
  learning.
\newblock \emph{arXiv preprint arXiv:2110.09348}, 2021.

\bibitem[Koishekenov et~al.(2023)Koishekenov, Vadgama, Valperga, and
  Bekkers]{geom_contrastive_loss}
Yeskendir Koishekenov, Sharvaree Vadgama, Riccardo Valperga, and Erik~J
  Bekkers.
\newblock Geometric contrastive learning.
\newblock In \emph{4th Visual Inductive Priors for Data-Efficient Deep Learning
  Workshop}, 2023.
\newblock URL \url{https://openreview.net/forum?id=cE4BY5XrzR}.

\bibitem[Misra and Maaten(2020)]{pirl}
Ishan Misra and Laurens van~der Maaten.
\newblock Self-supervised learning of pretext-invariant representations.
\newblock In \emph{Proceedings of the IEEE/CVF conference on computer vision
  and pattern recognition}, pages 6707--6717, 2020.

\bibitem[Nilsback and Zisserman(2008)]{flowers}
Maria-Elena Nilsback and Andrew Zisserman.
\newblock Automated flower classification over a large number of classes.
\newblock In \emph{Indian Conference on Computer Vision, Graphics and Image
  Processing}, Dec 2008.

\bibitem[Oord et~al.(2018)Oord, Li, and Vinyals]{contr_pred_coding}
Aaron van~den Oord, Yazhe Li, and Oriol Vinyals.
\newblock Representation learning with contrastive predictive coding.
\newblock \emph{arXiv preprint arXiv:1807.03748}, 2018.

\bibitem[Oquab et~al.(2023)Oquab, Darcet, Moutakanni, Vo, Szafraniec, Khalidov,
  Fernandez, Haziza, Massa, El-Nouby, et~al.]{dinov2}
Maxime Oquab, Timoth{\'e}e Darcet, Th{\'e}o Moutakanni, Huy Vo, Marc
  Szafraniec, Vasil Khalidov, Pierre Fernandez, Daniel Haziza, Francisco Massa,
  Alaaeldin El-Nouby, et~al.
\newblock Dinov2: Learning robust visual features without supervision.
\newblock \emph{arXiv preprint arXiv:2304.07193}, 2023.

\bibitem[Reviewer-kELi()]{common_stability_comment}
Reviewer-kELi.
\newblock Openreview comment.
\newblock \url{https://openreview.net/forum?id=rw2dkPZOoQ&noteId=kFoyEsqLRv}.
\newblock Accessed: 2024-03-04.

\bibitem[Richemond et~al.(2020)Richemond, Grill, Altch{\'e}, Tallec, Strub,
  Brock, Smith, De, Pascanu, Piot, et~al.]{byol_without_batch_statistics}
Pierre~H Richemond, Jean-Bastien Grill, Florent Altch{\'e}, Corentin Tallec,
  Florian Strub, Andrew Brock, Samuel Smith, Soham De, Razvan Pascanu, Bilal
  Piot, et~al.
\newblock Byol works even without batch statistics.
\newblock \emph{arXiv preprint arXiv:2010.10241}, 2020.

\bibitem[Richemond et~al.(2023)Richemond, Tam, Tang, Strub, Piot, and
  Hill]{BYOL_orthogonality}
Pierre~Harvey Richemond, Allison Tam, Yunhao Tang, Florian Strub, Bilal Piot,
  and Felix Hill.
\newblock The edge of orthogonality: a simple view of what makes byol tick.
\newblock In \emph{International Conference on Machine Learning}, pages
  29063--29081. PMLR, 2023.

\bibitem[Saunshi et~al.(2022)Saunshi, Ash, Goel, Misra, Zhang, Arora, Kakade,
  and Krishnamurthy]{understanding_contr_learn_2}
Nikunj Saunshi, Jordan Ash, Surbhi Goel, Dipendra Misra, Cyril Zhang, Sanjeev
  Arora, Sham Kakade, and Akshay Krishnamurthy.
\newblock Understanding contrastive learning requires incorporating inductive
  biases.
\newblock In \emph{International Conference on Machine Learning}, pages
  19250--19286. PMLR, 2022.

\bibitem[Smith et~al.(2023)Smith, Ortmann, Abbas-Aghababazadeh, Smirnov, and
  Haibe-Kains]{distribution_of_cosine_sim}
Ian Smith, Janosch Ortmann, Farnoosh Abbas-Aghababazadeh, Petr Smirnov, and
  Benjamin Haibe-Kains.
\newblock On the distribution of cosine similarity with application to biology.
\newblock \emph{arXiv preprint arXiv:2310.13994}, 2023.

\bibitem[Tian et~al.(2021)Tian, Chen, and Ganguli]{direct_pred}
Yuandong Tian, Xinlei Chen, and Surya Ganguli.
\newblock Understanding self-supervised learning dynamics without contrastive
  pairs.
\newblock In \emph{International Conference on Machine Learning}, pages
  10268--10278. PMLR, 2021.

\bibitem[Wang et~al.(2017)Wang, Xiang, Cheng, and Yuille]{normface}
Feng Wang, Xiang Xiang, Jian Cheng, and Alan~Loddon Yuille.
\newblock Normface: L2 hypersphere embedding for face verification.
\newblock In \emph{Proceedings of the 25th ACM international conference on
  Multimedia}, MM ’17. ACM, October 2017.
\newblock \doi{10.1145/3123266.3123359}.
\newblock URL \url{http://dx.doi.org/10.1145/3123266.3123359}.

\bibitem[Wang and Isola(2020)]{understanding_contr_learn}
Tongzhou Wang and Phillip Isola.
\newblock Understanding contrastive representation learning through alignment
  and uniformity on the hypersphere.
\newblock In \emph{International Conference on Machine Learning}, pages
  9929--9939. PMLR, 2020.

\bibitem[Zhang et~al.(2022)Zhang, Zhang, Zhang, Pham, Yoo, and
  Kweon]{simsiam_avoid_collapse}
Chaoning Zhang, Kang Zhang, Chenshuang Zhang, Trung~X Pham, Chang~D Yoo, and
  In~So Kweon.
\newblock How does simsiam avoid collapse without negative samples? a unified
  understanding with self-supervised contrastive learning.
\newblock \emph{arXiv preprint arXiv:2203.16262}, 2022.

\bibitem[Zhang et~al.(2020)Zhang, Li, and Zhang]{spherical_embeddings}
Dingyi Zhang, Yingming Li, and Zhongfei Zhang.
\newblock Deep metric learning with spherical embedding.
\newblock \emph{Advances in Neural Information Processing Systems},
  33:\penalty0 18772--18783, 2020.

\bibitem[Zhang et~al.(2023)Zhang, Kang, Hooi, Yan, and Feng]{long-tail-survey}
Yifan Zhang, Bingyi Kang, Bryan Hooi, Shuicheng Yan, and Jiashi Feng.
\newblock Deep long-tailed learning: A survey.
\newblock \emph{IEEE Transactions on Pattern Analysis and Machine
  Intelligence}, 2023.

\bibitem[Zimmermann et~al.(2021)Zimmermann, Sharma, Schneider, Bethge, and
  Brendel]{latent_inversion}
Roland~S Zimmermann, Yash Sharma, Steffen Schneider, Matthias Bethge, and
  Wieland Brendel.
\newblock Contrastive learning inverts the data generating process.
\newblock In \emph{International Conference on Machine Learning}, pages
  12979--12990. PMLR, 2021.

\end{thebibliography}
